\documentclass{article}
\usepackage{log_2023}
\usepackage{microtype}
\usepackage{graphicx}
\usepackage{subfigure}
\usepackage{booktabs} 

\usepackage{algorithm}
\usepackage{algorithmic}
\usepackage{amsmath}

\usepackage{amsmath}
\usepackage{amssymb}
\usepackage{mathtools}
\usepackage{amsthm}
\usepackage{listings}

\theoremstyle{plain}
\newtheorem{theorem}{Theorem}[section]
\newtheorem{proposition}[theorem]{Proposition}

\theoremstyle{definition}

\theoremstyle{remark}

\usepackage[textsize=tiny]{todonotes}

\title{Cycle Invariant Positional Encoding for Graph Representation Learning}

\author{
Zuoyu Yan \\
Wangxuan Institute of Computer Technology \\
Peking University \\
\texttt{yanzuoyu3@pku.edu.cn}
\And 
Tengfei Ma \\
Department of Biomedical Informatics \\
Stony Brook University \\
\texttt{Tengfei.Ma@stonybrook.edu}
\AND 
Liangcai Gao$^*$ \\
Wangxuan Institute of Computer Technology \\
Peking University \\
\texttt{glc@pku.edu.cn}
\And
Zhi Tang \\
Wangxuan Institute of Computer Technology \\
Peking University \\
\texttt{tangzhi@pku.edu.cn}
\AND
Chao Chen$^*$\\
Department of Biomedical Informatics \\
Stony Brook University \\
\texttt{chao.chen.1@stonybrook.edu}
\And
Yusu Wang\thanks{Correspondence to Yusu Wang, Chao Chen, and Liangcai Gao} \\
Hal{\i}c{\i}o\u{g}lu Data Science Institute \\
University of California \\
\texttt{yusuwang@ucsd.edu}
}

\begin{document}
\maketitle






\begin{abstract}
Cycles are fundamental elements in graph-structured data and have demonstrated their effectiveness in enhancing graph learning models. To encode such information into a graph learning framework, prior works often extract a summary quantity, ranging from the number of cycles to the more sophisticated persistence diagram summaries. However, more detailed information, such as which edges are encoded in a cycle, has not yet been used in graph neural networks. In this paper, we make one step towards addressing this gap, and propose a structure encoding module, called CycleNet, that encodes cycle information via edge structure encoding in a permutation invariant manner. To efficiently encode the space of all cycles, we start with a cycle basis (i.e., a minimal set of cycles generating the cycle space) which we compute via the kernel of the 1-dimensional Hodge Laplacian of the input graph. To guarantee the encoding is invariant w.r.t. the choice of cycle basis, we encode the cycle information via the orthogonal projector of the cycle basis, which is inspired by BasisNet proposed by Lim et al. 
We also develop a more efficient variant which however requires that the input graph has a unique shortest cycle basis. To demonstrate the effectiveness of the proposed module, we provide some theoretical understandings of its expressive power. Moreover, we show via a range of experiments that networks enhanced by our CycleNet module perform better in various benchmarks compared to several existing SOTA models. 
\end{abstract}

\section{Introduction}
\label{sec:intro}

The incorporation of structural information has been shown beneficial to graph representation learning~\cite{zhou2020graph, wu2020comprehensive}. In recent years, message passing neural networks (MPNNs) have become a popular architecture for graph learning tasks. It has been shown \cite{xu2018powerful, morris2019weisfeiler} that in terms of differentiating graphs, MPNNs have the same power as the well-known Weisfeiler-Lehman (WL) graph isomorphism test. WL-tests in fact take as inputs two graphs with node features (called ``colors" or ``labels" in the literature). As the initial node features become more representative, the power of WL-tests also increases. If initial node features are simple summaries that can be computed from the one-ring neighborhood of each point (e.g., the degree of each node), then it is known that the resulting MPNNs cannot detect structures such as cycles which could be important for application domains such as biology~\cite{koyuturk2004efficient}, chemistry~\cite{deshpande2002automated, jin2018junction} (e.g., rings), and sociology~\cite{koyuturk2004efficient} (e.g., the triadic closure property). 

Several pieces of work have been developed to enhance GNNs' ability in encoding cycle-like structures. These approaches can be loosely divided into two categories: (1) methods that extract a summary quantity, ranging from the number of cycles~\cite{bouritsas2022improving} to the more sophisticated persistence diagram summaries~\cite{hofer2017deep, yan2021link} to improve graph representation learning; (2) methods that perform message passing among high-order cycle/topology-related structures~\cite{bodnar2021weisfeilera, bodnar2021weisfeiler}. 
However, these methods often suffer from high computational costs, and more detailed information, such as which edges are encoded in a cycle, is not yet contained in these models, which may limit their representation power. For example, just using simple summaries, such as the number and lengths of cycles, is not sufficient to differentiate a well-known pair of strongly regular graphs: the $4 \times 4$ Rook Graph and the Shrikhande Graph. In contrast, as shown by the proof of Theorem~\ref{th:cyclenet} in the appendix, our CycleNet can differentiate them. 

The high level goal of this paper is to develop efficient and effective ways to encode more detailed cycle information. In particular, much like using Laplacian eigenfunctions to provide {\bf positional encoding} for nodes in a graph, we wish to develop {\bf edge structure encoding} which intuitively provides the position of each edge in terms of the entire cycle space for a graph. In particular, we propose {\bf CycleNet} which do so via (the kernel space of) the 1-dimensional Hodge Laplace operator $\Delta_1$ of the graph. Indeed, from Hodge theory~\cite{jiang2011statistical, lim2020hodge}, we know that the space of 1-chains (in real coefficients) can be decomposed into two subspaces: the cycle space, which is the kernel space of the \textit{1-Hodge Laplacian $\Delta_1$}, and the gradient space, which stores the distinction between node signals. 
We will use an orthonormal \textit{cycle basis $\Gamma$} computed from the kernel of $\Delta_1$ to represent the cycle space (note, a cycle basis is simply a minimal set of cycles spanning the cycle space). The CycleNet will compute an edge structure encoding vector for each graph edge based on $\Gamma$. Note that we require the CycleNet to be both (1) equivariant w.r.t. the permutation of edge orders, and (2) invariant to the choice of cycle basis we use (i.e., for different cycle bases of the same cycle space, the model should produce the same output). To this end, we leverage the idea from~\cite{lim2022sign} and encode the cycle information via the orthogonal projector of the cycle basis, which allows a universal approximation for functions satisfying the above two conditions. By combining CycleNet with graph learning models, we can effectively encode the cycle information into graph learning models. 


To further improve efficiency, as well as to make the edge encoding more intuitive, we also propose CycleNet-PEOI, a variant of CycleNet that assumes the input graph to have a unique shortest cycle basis (SCB). An SCB is a cycle basis whose total length/weight of all cycles in a basis is minimal. Instead of basis invariance, here we fix the basis to be the shortest cycle basis (in $\mathbb{Z}_2$ coefficients), and hence we only need to guarantee order invariance, that is, the output should be invariant to the permutation order of these basis cycles. This allows us to have a simpler architecture to represent such functions. The assumption of a unique SCB is strong, but seems reasonable for certain datasets, such as molecular graphs where cycles correspond to chemical substructures like Benzene rings or weighted graphs where the weights of cycles are different (see Section~\ref{sec:unique} in the appendix). 

The contributions of our work are summarized as follows: 
\begin{itemize}
\item In Section~\ref{sec:cyclenet}, we propose a novel edge structure encoding module, CycleNet, which encodes, for each edge, its ``position'' in the cycle spaces of the graph. The module encodes the cycle information in a permutation invariant and basis invariant manner.
\item We also propose CycleNet-PEOI, a variant of CycleNet based on the theory of algebraic topology. The encoding only requires order invariance, making it significantly more efficient.
\item We provide theoretical analyses to establish the expressive power of the proposed modules. Additionally, we conduct comprehensive empirical evaluations on diverse benchmarks, encompassing both synthetic and real-world datasets. Through these experiments, we showcase the remarkable representation capabilities and algorithmic efficiency of our proposed modules.

\end{itemize}

\section{Related Works}

\paragraph{Cycle-related Graph Representation Learning. }

Existing works encode cycle-related information mainly from two perspectives. The first one is to encode a summary quantity. These works include~\cite{bouritsas2022improving, yan2023efficiently}, which use the number of substructures as augmented node features,~\cite{yan2022cycle, zhangge2023fl}, which extract the semantic information of cycles, and~\cite{hofer2017deep, hu2019topology, zhao2019learning, chen2021topological, zhao2020persistence,yan2021link, horn2021topological, yan2022neural} that introduce the persistent homology~\cite{edelsbrunner2022computational, cohen2009extending}, a summary of cycle information as augmented features. Most of these do not look at detailed cycle compositions and edges' relation to them. 

The second category is to enhance the message passing function with high-order cycle-related structures. For example, \cite{bodnar2021weisfeiler, bodnar2021weisfeilera} propose a new type of message passing function based on the simplicial/cell complexes. \cite{goh2022simplicial, keros2022dist2cycle} extend the framework to more downstream tasks. 
We note that if the input is a graph $G = (V, E)$ without higher-order simplices/cells information, often a choice has to be made regarding how to construct high-dimensional cells/simplices. For example, one can take cliques in the input graph to form high-dimensional simplices, or use a set of cycles as the boundary of 2-cells. However, the choice for the latter often may not be canonical (i.e., even for the same graph different choices can be made). 


\paragraph{Positional Encodings on Graphs.}
To leverage the spectral properties of graphs, many works~\cite{dwivedi2020benchmarking, lim2022sign, dwivedi2021graph, kreuzer2021rethinking, mialon2021graphit} introduce the eigenvectors of the graph Laplacian as augmented node features. 
Other approaches introduce positional encodings such as random walks~\cite{li2020distance}, diffusion kernels~\cite{feldman2022weisfeiler}, shortest path distance~\cite{ying2021transformers}, and unsupervised node embedding methods~\cite{wang2022equivariant}.
Our work can be viewed as an structural encoding for edges in a graph via their ``position" in the cycle space. 

\paragraph{Hodge Laplacians and Graph Neural Networks.}
Previously, Hodge Laplacian has been widely used in signal processing, such as flow denoising~\cite{schaub2018flow}, flow sampling~\cite{jia2019graph}, and topology inference~\cite{barbarossa2020topological, barbarossa2018learning}.  Recent works have combined it with GNNs to tackle flow interpolation~\cite{roddenberry2019hodgenet} and mesh representation learning~\cite{smirnov2021hodgenet}. However, these methods only encode the Hodge Laplacians as augmented structural features, and seldom focus on the information of cycles. In this paper, we propose novel ways to encode the cycle space of the Hodge Laplacians. Theoretical and empirical results demonstrate the effectiveness of our proposed framework.

\section{Preliminaries}
\label{sec:pre}
\textbf{Hodge Theory.} We present a brief overview of Hodge decomposition in the context of simple graphs and refer interested readers to~\cite{lim2020hodge, jiang2011statistical} for further details. 

Let $G = (V,E)$ be a simple graph with node set $V = \{1,\ldots,n\}$ and edge set $E = \{e_1,\ldots,e_m\}$. The adjacency matrix of $G$ is denoted by $A \in \mathbb{R}^{n \times n}$, where $A_{ij} = 1$ if $(i,j) \in E$ and $A_{ij} = 0$ otherwise. $m$ denotes the number of edges and $n$ denotes the number of nodes. The incidence matrix of $G$ is denoted by $B \in \mathbb{R}^{n \times m}$ and defined as follows:
\begin{equation}
 \label{eq:inc}
 B_{ij} = \left\{
 \begin{aligned}
 -1 & \text{ if } e_j = (i, k) \text{ for some } k \in V, \\
 1 &\text{ if } e_j = (k, i) \text{ for some } k \in V, \\
 0  & \text{ otherwise. }\\
 \end{aligned}
 \right.
 \end{equation}
For undirected graphs, the choice of direction for an edge in the incidence matrix is arbitrary and does not affect subsequent definitions. Using topological language, $B$ essentially is the boundary matrix from the $1$-chain group to the $0$-chain group. 

There exist various techniques to extract node-level information from graphs. One widely adopted approach is to utilize the eigenvectors of the graph Laplacian $\Delta_0$. It is defined as $\Delta_0 = D - A$, where $D$ is the diagonal matrix of node degrees and $A$ is the adjacency matrix. Alternatively, $\Delta_0$ can also be computed as $\Delta_0 = BB^T$.

The Hodge Laplacian is a high-order generalization of the graph Laplacian, and serves as a graph shift operator defined on edges, specifically $\Delta_1 = B^TB \in \mathbb{R}^{m\times m}$\footnote{This definition is specific for simple graphs, not for general simplicial complexes.}. Unlike the graph Laplacian, which can be used for signal processing for functions defined on graph nodes, $\Delta$ is an operator for functions defined on graph edges (note that a real valued function on the set of edges can be viewed as a vector in $\mathbb{R}^m$). $\Delta_1$ intuitively measures the \textit{conservatism} of edges~\cite{barbarossa2020topological}. Edges can be classified into two types: conservative and non-conservative. Conservative edges are referred to as \textit{gradient}, as they are induced by measuring the distinction of nodes. Conversely, non-conservative edges are defined as \textit{divergence-free} or \textit{harmonic}\footnote{They are different in graphs that contain high-order simplices/cells, but are the same in simple graphs.}, as they are naturally composed of cycles.

In the context of simple graphs (viewed as a 1-dimensional simplicial complex), there is an orthogonal direct sum decomposition of $\mathbb{R}^m$ 
according to the Hodge Decomposition theory: 
\begin{equation}
\label{eq:hodge}
\mathbb{R}^m = \text{ker}(\Delta_1) \bigoplus \text{Im}(B^T)
\end{equation}
Here, ker$(\Delta_1) = ker(B)$ denotes the kernel space of the Hodge Laplacian $\Delta_1$, and it turns out that it is in fact isomorphic to the 1-dimensional cycle space of $G$ (viewed as a 1-D simplicial complex) w.r.t. real coefficients \cite{lim2020hodge}. 
Im$(B^T)$ represents the image space of the incidence matrix $B^T$, which reflects the distinction of node information. It is worth noting that the Hodge Decomposition can be more generally defined on graphs that contain high-order simplices/cells, whereas we focus on simple graphs that only consist of nodes and edges.

\textbf{Shortest Cycle Basis.} We provide a brief overview of the theory of shortest cycle basis, and refer the readers to~\cite{munkres2018elements, dey2022computational, edelsbrunner2022computational} for a more comprehensive understanding. Let $G = (V,E)$ be an input graph. In this paper, a \textit{cycle} is defined as a subgraph of $G$ in which each vertex has a degree of 2. We can describe a cycle using an incidence vector $C$ indexed on $E$. The $e$-th index of $C$ is $1$ if $e$ is an edge of the cycle, and $0$ otherwise. The incidence vectors of all cycles form a vector space $Z_G$ over $\mathbb{Z}_2$, 
which is called the \textit{cycle space} of $G$. The dimension of the cycle space is $g = m - n + 1$, where $g$ is the \textit{Betti number}. The \textit{cycle basis} is a set of linearly independent cycles that span the cycle space, i.e., any cycle in $Z_G$ can be expressed as the modulo-2 sum of cycles in the basis. A cycle basis $\{\gamma_1, \gamma_2, \dots, \gamma_g\}$ can be described by a \textit{cycle incidence matrix} $X \in R^{m \times g}$, which is formed by combining of the incidence vector of cycles in the cycle basis. Specifically, the $i$-th column vector $X_i$ of $X$ corresponds to the incidence vector of the $i$-th cycle $\gamma_i$.

We define the weight of a cycle as the number of edges it contains, and the weight of the cycle basis as the sum of the weights of its constituent cycles. The \textit{shortest cycle basis} (SCB) is defined as a cycle basis of the minimum weight.


\begin{figure*}[btp]
\vskip -0.2in
\begin{center}
\centerline{\includegraphics[width=\columnwidth]{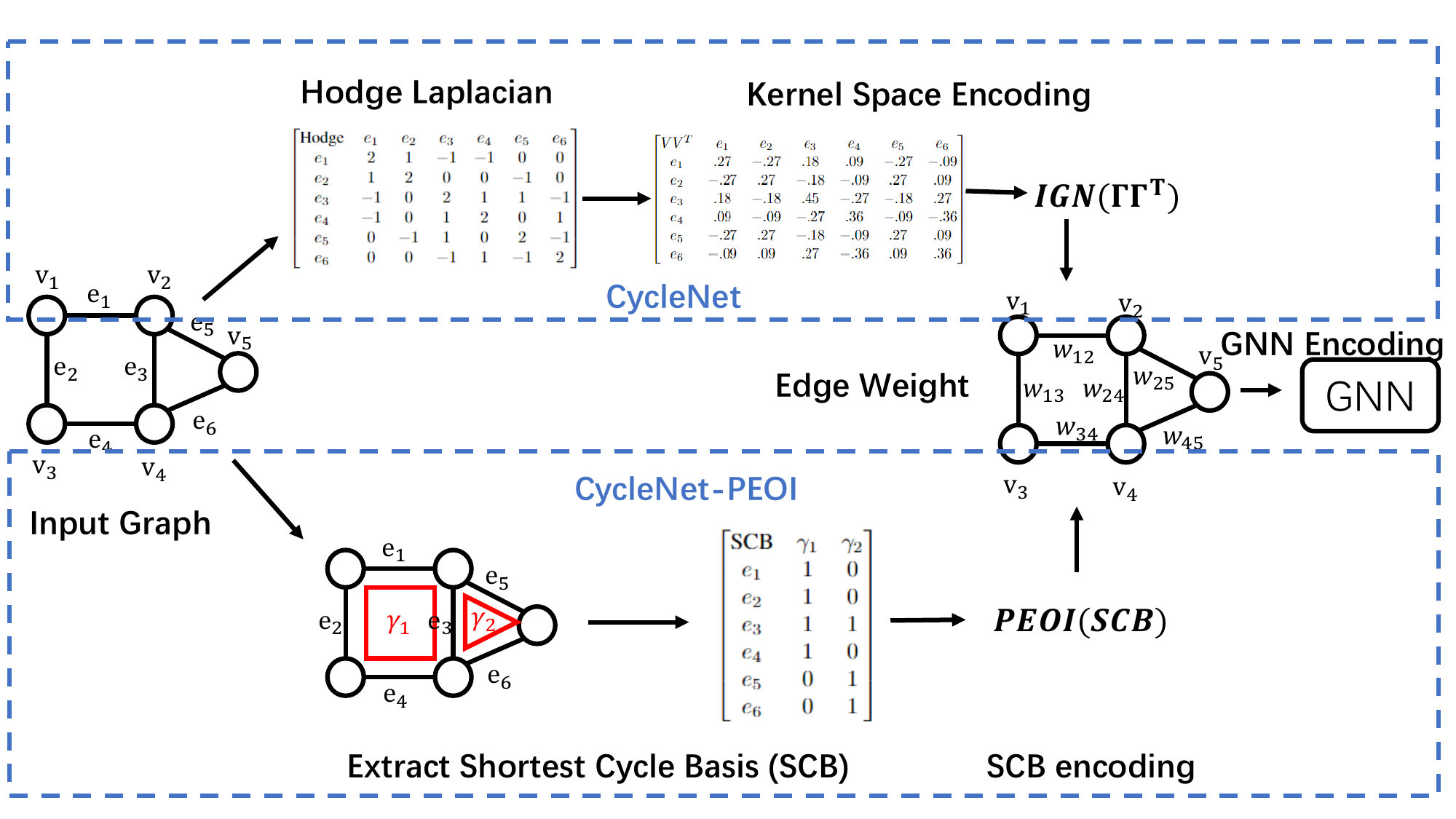}}
\caption{The framework of CycleNet. We either adopt the upper branch (CycleNet) or the lower branch (CycleNet-PEOI) as the framework.}
\label{fig:cyclenet}
\end{center}
\vskip -0.3in
\end{figure*}

\section{CycleNet}
\label{sec:cyclenet}
In this section, we present the framework of our proposed module. We begin by describing the framework, and then introduce how to compute functions that conform to the symmetries of the cycle space. We also provide some theoretical findings on the expressive power of the module.

\subsection{CycleNet}

The proposed module, which is illustrated in Figure~\ref{fig:cyclenet}, encodes the cycle information via edge structure encoding. Let $h^t_i$ denote the embedding for node $i$ in the $t$-th iteration. At the start of the process (i.e., $t=0$), the embedding is initialized with the intrinsic attributes of the nodes. In each subsequent iteration $(t+1)$, it is updated as:

\begin{equation}
\label{eq:cyclenet}
h_i^{t+1} = W_1^t(h_i^t, \sum_{j \in N(i)}W_2^t(h_i^t, h_j^t, e_{ij}, s_{ij}))
\end{equation}
where $W_1^t$ and $W_2^t$ are two trainable matrices, $N(i) = \{j \in V|(i, j) \in E\}$ denotes the neighborhood of $i$, and $s_{ij}$ is the {\bf structural embedding} of edge $(i,j)$ to capture informaiton about the cycle space. We will introduce its computation in the next section. 
The following proposition states that CycleNet can differentiate any pair of graphs with different edge structural encoding.


\begin{proposition}
\label{pro:edge}
Denote the set of edge structural embedding as $S \in R^{m \times d}$, where $s_e \in R^d$ represents the edge structural embedding for edge $e$. If a pair of non-isomorphic graphs have  distinct $S$, then there exists a CycleNet that utilizes $S$ as the edge structural embedding, capable of distinguishing between them. 
\end{proposition}

Given two graphs $G_1 = (V_1, E_1)$ and $G_2 = (V_2, E_2)$, let $F_s$ denote the set of all bijective mappings from $E_1$ to $E_2$. If they have different $S$, then for each $f_s \in F_s$, there must exist at least one edge $e_1 = (u_1, v_1) \in E_1$ and its paired edge $f_s(e_1) = e_2 = (u_2, v_2) \in E_2$ such that $s_{e_1} \neq s_{e_2}$. If Equation~\ref{eq:cyclenet} satisfies that (1) $W_1^t$ and $W_2^t$ are injective functions; (2) the graph-level readout function is injective to the multiset of node features, then CycleNet can differentiate $G_1$ and $G_2$ following the proof of Theorem 3 from~\cite{xu2018powerful}. Notice that the proof is based on the assumption that the node features and the edge features are from a countable set.

\subsection{Encoding the cycle space}
Recall that the goal of the paper is to devise efficient and effective frameworks to encode detailed cycle information. In this section, we investigate edge structure encoding approaches that enable the determination of the location of each edge with respect to the entire cycle space.

\subsubsection{Basis invariant functions of the cycle space}
We present a framework for computing functions that respect the basis variance of the cycle space of the Hodge Laplacian. Specifically, for the input graph, we extract the eigenvectors $\Gamma \in \mathbb{R}^{m \times g}$ of the kernel space of the Hodge Laplacian, where $m$ is the number of edges and $g$ is the Betti Number. According to the theory of Hodge decomposition, the eigenvectors form an orthonormal cycle basis that spans the cycle space. The structure encoding $f$ should be invariant to the right multiplication by any orthogonal matrix $Q$. Additionally, it should be equivariant to permutations along the row axis. Formally, we require that $f(\Gamma Q) = f(\Gamma)$ for any $Q \in \mathrm{O}(g)$, where $\mathrm{O}(g)$ denotes the set of $g \times g$ orthogonal matrices, and $f(P\Gamma) = Pf(\Gamma)$ for any $P \in \Pi[m]$, where $\Pi[m]$ denotes the set of $m \times m$ permutation matrices. 

Such ``left permutation equivariance" and ``right basis invariance" requirements are exactly the setup of Basisnet proposed in \cite{lim2022sign}. Specifically, BasisNet universally approximates all basis invariant functions on the eigenspace. Following \cite{lim2022sign}, we map the eigenvectors to the orthogonal projector of its column space: $\Gamma \rightarrow \Gamma \Gamma^T$, which is $O(d)$ invariant and retains all the information. To preserve the permutation equivariance along the row axis, the proposed model $f_{basis}:\mathbb{R}^{m \times m} \rightarrow \mathbb{R}^{m \times d}$ should satisfy $f_{basis}(P\Gamma\Gamma^TP^T) = Pf_{basis}(\Gamma\Gamma^T)$ for any permutation matrix $P$. We use the invariant graph network (IGN)~\cite{maron2018invariant}, a graph learning model capable of encoding permutation equivariant operations, to parameterize this mapping. The final model is presented below:
\begin{equation}
\label{eq:basis}
h(\Gamma) = \text{IGN}(\Gamma\Gamma^T)
\end{equation}

$\text{IGN}(\Gamma\Gamma^T)$ universally approximates any left permutation equivariant and right basis invariant functions over $\Gamma$, requiring the use of high order tensors (with order depending on $m$) \cite{maehara2019simple, maron2019universality, keriven2019universal}. This is rather expensive. In practice, we follow the practice of \cite{lim2022sign} and only use 2-IGN. 
We also note that in BasisNet of \cite{lim2022sign}, each eigenspace spanned by eigenvectors corresponding to the same eigenvalue of a linear operator (which is taken as the 0-th Laplacian $\Delta_0$ in their paper) requires a separate IGN of the form as in Eqn (\ref{eq:basis}); that is, an IGN needs to be constructed for every eigenvalue. In practice, one has to take the union of the eigenspace w.r.t. a range (interval) of eigenvalues to use a shared IGN (as otherwise, there will be an infinite number of IGNs needed theoretically). Our setting is much simpler as we only need to consider the kernel space, which is the space spanned by all eigenvectors of $\Delta_1$ corresponding to eigenvalue $0$. 

We then abuse notation slightly and use \emph{CycleNet} to also refer to the instantiation of Eqn (\ref{eq:cyclenet}) where the edge structural encoding for the $i$th edge is taken as $IGN(\Gamma\Gamma^T) [i]$; see the top row in Figure~\ref{fig:cyclenet}.

\subsubsection{Permutation equivariant and order invariant (PEOI) functions of the cycle basis} \label{subsub:PEOI}
The encoding based on the Hodge Laplacian, although powerful, presents two issues. 
First, $\text{IGN}(\Gamma\Gamma^T)$ takes a $m\times m$ matrix $\Gamma\Gamma^T$ as input, where $m$ is the number of edges in the input graph, and is thus expensive.
Furthermore, as the cycle basis passes through the basis invariant encoding, it becomes hard to decipher which cycles are being parameterized and contribute to graph representation learning. To this end, we develop an invariant of the aforementioned CycleNet, which we call \emph{CycleNet-PEOI} (lower row in Figure~\ref{fig:cyclenet}), using the so-called SCB of the input graph. In real-world benchmarks, the SCB often contains essential components such as ternary relationships and benzene rings. We will also theoretically show that it contains valuable structural information in Section~\ref{subsec:the}.

More specifically, given a graph $G = (V, E)$, let $\{\gamma_1, \ldots, \gamma_g\}$ be a SCB (see Section~\ref{sec:pre}) of $G$. 
We assume that the SCB of the input graph is unique -- this does not hold for general graphs. Nevertheless, for many real-world graphs (e.g., those representing chemical compounds), important structures such as ternary relationships and benzene rings seldom overlap with other cycles. In addition, in weighted graphs, the different weights of cycles guarantee a unique SCB.

\paragraph{Remark.} As we will see below, our module can be defined and used even when SCB is not unique -- in such case, for the same graph, we might obtain different edge structural encoding depending on the choice of SCBs. 
Note that another way to encode the SCB is to fill these cycles with 2-cells and apply the cellular message passing network of \cite{bodnar2021weisfeiler}. The same issue exists for this approach as well.

%

Now given the SCB $\{\gamma_1, \ldots, \gamma_g\}$, consider its corresponding cycle incidence matrix: $X\in \mathbb{R}^{m \times g}$ defined such that the $i$th column $X_i$ of $X$ is a $m$-D 0/1 vector where $X_i[j] = X[i][j] = 1$ if and only if edge $j$ is in the cycle $\gamma_i$; that is, $X_i$ indicates the set of edges in cycle $\gamma_i$. 
Our goal is to compute a (edge-encoding) function $f: \mathbb{R}^{m \times g} \rightarrow \mathbb{R}^{m \times d}$, which is permutation equivariant along the row axis, while being order invariant to the permutation of columns -- the latter is because our edge encoding should not depend on the order (permutation) of the cycles in a SCB. We refer to this symmetry as \textit{permutation equivariant and order invariant (PEOI)}, which exists if and only if the following two conditions hold:

\begin{itemize}
\item For any $m \times m$ permutation matrix $P_1 \in \Pi [m]$, we have $P_1f(X) = f(P_1X)$.
\item For any $g \times g$ permutation matrix $P_2 \in \Pi [g]$, we have $f(X) = f(XP_2)$.
\end{itemize}

\textbf{The function that satisfies the PEOI.}
Note that we do not have universal approximation results for PEOI functions -- even if the universal approximation holds, it is likely that the latent dimension might depend on $m$ (similar to the universal approximation of DeepSet for permutation invariant functions \cite{NIPS2017deepsets}). Denote the function $F: \mathbb{R}^{m \times g} \rightarrow \mathbb{R}^{m \times d}$. For the $i$-th row of $F(X)$, there exists functions, $\rho_1$, $\rho_2$, and $\rho_3$ that satisfies:

\begin{equation}
\label{eq:peoi}
F(X)[i] = \rho_3(\sum_{k \in [g]} \rho_2(X[i][k], \sum_{j \in [m], j \neq i} \rho_1 (X[i][k], X[j][k]])  ))
\end{equation}

In particular, the continuous functions $\rho_1: \mathbb{R}^2 \to \mathbb{R}^a$, $\rho_2: \mathbb{R}^{a + 1} \to \mathbb{R}^b$ and $\rho_3: \mathbb{R}^b \to \mathbb{R}^d$ above will be approximated by parametrized MLPs $\text{MLP}_1$, $\text{MLP}_2$, and $\text{MLP}_3$. 
We note that compared to the CycleNet embedding using IGN, if we choose constant latent dimension $a$ and $b$, and assuming the complexity of each MLP is bounded by a constant, then the total model complexity is bounded by a constant and the computation is only linear in $m$. Nevertheless, in practice, as $a$, $b$ and each of the $\text{MLP}_i$, $i\in \{1, 2, 3\}$ is of bounded size, this model is much more efficient than IGN (which is $\Omega(m^2)$ due to its input is a matrix of size $m\times m$).

\subsection{Theoretical analysis}
\label{subsec:the}

\textbf{Expressiveness of CycleNet.} We assess the expressiveness of the proposed model by evaluating its ability to differentiate between structurally distinct graphs, or \textit{non-isomorphic} graphs. The definition of isomorphism and the proofs of the theorems are provided in the appendix.

\begin{theorem}
\label{th:cyclenet}
CycleNet is strictly more powerful than 2-WL, and can distinguish certain pairs graphs that are not distinguished by 3-WL.
\end{theorem}

This denotes that the basis invariant function boost the representation power of CycleNet.




\textbf{Comparison with existing works.} The two theorems below demonstrate that the proposed PEOI encoding of the cycle incidence matrix is capable of extracting more informative features compared to models~\cite{bouritsas2022improving} using the length of cycles or models~\cite{yan2022neural, zhao2020persistence} that use the extended persistence diagrams (EPDs)~\cite{cohen2009extending} as augmented features. The ``EPD" here denotes the 1D EPD corresponding to cycles. 

\begin{theorem}
\label{th:number}
If choosing the same set of cycles. The PEOI encoding of the cycle incidence matrix is more powerful than using its number in terms of distinguishing non-isomorphic graphs.
\end{theorem}

\begin{theorem}
\label{th:ph}
If choosing the same set of cycles. The PEOI encoding of the cycle incidence matrix can differentiate graphs that cannot be differentiated by the extended persistence diagram. If adding the filter function to the cycle incidence matrix, the PEOI encoding is more powerful than the extended persistence diagram in terms of distinguishing non-isomorphic graphs.
\end{theorem}

\textbf{Choice of the cycle incidence matrix.} We demonstrate that the SCB contains valuable structural information in terms of differentiating non-isomorphic graphs. It is worth noting that most existing works compare their frameworks with 2-WL and 3-WL, whereas the SCB can distinguish graphs that 4-WL cannot distinguish. 

\begin{theorem}
\label{th:cyclenet_scb}
Using the length of shortest cycle basis as the edge structural embedding can distinguish certain pair of graphs that are not distinguished by 3-WL, as well as pair of graphs that are not distinguished by 4-WL.
\end{theorem}

\section{Experiments}

In this section, we evaluate the proposed framework from two perspectives: (1) In Section~\ref{subsec:syn}, we assess whether the framework can extract the cycle information effectively and preserve the expressive power; (2) in Section~\ref{subsec:bench}, we examine whether the incorporation of cycle information contributes to the improvement of the downstream tasks. The code is available at \url{https://github.com/pkuyzy/CycleNet}.

\textbf{Baselines.} We adopt various GNN models as baselines to evaluate the effectiveness of our proposed framework. These models include GIN~\cite{xu2018powerful}, GCN~\cite{kipf2017semi} and GAT~\cite{velivckovic2018graph}, which are MPNN models; PPGN~\cite{maron2019provably}, a high-order GNN which is as powerful as the 3-WL; SCN~\cite{ebli2020simplicial}, SCCONV~\cite{bunch2020simplicial}, CWN~\cite{bodnar2021weisfeiler}, SAT~\cite{goh2022simplicial}, and Dist2Cycle~\cite{keros2022dist2cycle}, which introduce the cycle information using the simplicial complex or the cell complex; SignNet and BasisNet~\cite{lim2022sign}, which introduces encodings that respect the sign variance or the basis variance of the eigenspaces. We combine our CycleNet and CycleNet-PEOI modules with backbone models (named ``backbone+CycleNet") and report the results. Note that frameworks~\cite{ebli2020simplicial, bunch2020simplicial, bodnar2021weisfeiler, goh2022simplicial, keros2022dist2cycle} that fill cycles with 2-cells cannot be combined with CycleNet-PEOI since there is no cycle in the 
corresponding simplicial/cell complexes.

\subsection{Synthetic benchmarks}
\label{subsec:syn}

\textbf{Datasets.} To evaluate the expressiveness of the proposed module, we use the strongly regular (SR) graph dataset from~\cite{bodnar2021weisfeilera} as the benchmark, which contains 227 graphs with different isomorphic types and cannot be distinguished by 3-WL test. Following the settings of~\cite{bodnar2021weisfeiler}, we use the cosine distance between the extracted embeddings of a pair of graphs to determine whether they are isomorphic or not. Additionally, we generate the Cai-F\"urer-Immerman (CFI) graphs~\cite{cai1992optimal} based on the proof of Theorem~\ref{th:cyclenet_scb}, consisting of 200 graphs. They are generated from two isomorphism types by randomly permuting the node sequence. We categorize the graphs into two classes based on their isomorphism types and use classification accuracy as the evaluation metric. Moreover, we measure the running time (both training and inference) per epoch to evaluate the algorithmic efficiency of the model. 

To evaluate the effectiveness of the proposed model in terms of extracting cycle information, we generate a point cloud dataset, which is sampled from several small cycles whose centers are on a large cycle. The task is to predict the Betti number and the extended persistence diagrams (EPDs)~\cite{cohen2009extending} of these nontrivial cycles. These attributes are theoretically significant in the context of computational topology~\cite{dey2022computational, edelsbrunner2022computational}, in which the Betti number denotes the dimension of the cycle space, and the EPD is a topological summary of the cycles. We use the Mean Absolute Error (MAE) for the Betti number, and the Mean Square Error (MSE) for the EPD as the evaluation metric.

\textbf{Results.} The results are presented in Table~\ref{tab:non} and Table~\ref{tab:cycle}. In terms of expressiveness, the proposed model can differentiate the CFI graphs and the strongly regular graphs, which is consistent with our theoretical findings. In addition, the proposed model outperforms the baselines in terms of predicting the Betti number and the EPDs. It empirically justifies that CycleNet can extract useful cycle information. We are surprised to observe that SignNet can differentiate the 4-CFI graphs, although it fails to distinguish the 3-CFI graphs. However, this result is not convincing since SignNet violates the basis variance property, i.e.,  different eigenvectors from the same eigenspace must produce the same output, which is not guaranteed by SignNet's approach.

Regarding algorithmic efficiency, our proposed model, CycleNet, exhibits a slightly slower performance compared to the GIN backbone model, while outperforming other baseline models by a significant margin. This result indicates that the proposed framework achieves a desirable balance between its expressiveness and computational efficiency. 

\begin{table}
    \centering
    \caption{Experiments on expressiveness, Accuracy and Seconds per epoch}  
    \label{tab:non}
    \scalebox{0.9}{
    \begin{tabular}{|l|cc|cc|cc|}
    \hline\noalign{\smallskip}
    & \multicolumn{2}{c|}{3-CFI} & \multicolumn{2}{c|}{4-CFI} & \multicolumn{2}{c|}{SR} \\
    \hline\noalign{\smallskip}
    Method & Acc & Sec/epo& Acc & Sec/epo& Acc & Sec/epo\\
    \noalign{\smallskip}\hline\noalign{\smallskip}
    GIN & 0.50 & 0.29 & 0.64 & 0.46& 0.50 & 0.64\\
    SignNet & 0.50 & 1.11 & 1 & 5.43& 1 & 2.59\\
    PPGN & 0.50 & 0.36& 0.50 & 0.57 & 0.50 & 1.67\\
    CWN & 1 & 4.28& 0.50 & 4.94 & 1 & 19.71\\
    \noalign{\smallskip}\hline\noalign{\smallskip}
    GIN+CycleNet & 1 & 0.56 & 0.50 & 1.41& 1 & 2.12\\
    GIN+CycleNet-PEOI & 1 & 0.45& 1 & 1.11& 1 & 2.08\\
    \noalign{\smallskip}\hline\noalign{\smallskip}
    \end{tabular}}
    \vskip -0.2in
\end{table}

\begin{table}
    \centering
    \caption{Experiments on approximating the Betti Number and the extended persistence diagram, Regression Error and Seconds per epoch}  
    \label{tab:cycle}
    \scalebox{0.9}{
    \begin{tabular}{|l|cc|cc|}
    \hline\noalign{\smallskip}
     & \multicolumn{2}{c|}{Betti Num} & \multicolumn{2}{c|}{EPD}\\
    \hline\noalign{\smallskip}
    Method & Error & Sec/epo& Error & Sec/epo\\
    \noalign{\smallskip}\hline\noalign{\smallskip}
    GIN & 0.273$\pm$0.059 & 1.22 & 0.214$\pm$0.004 & 1.33\\
    SignNet & 0.112$\pm$0.009 & 22.31& 0.182$\pm$0.010 & 23.66\\
    CWN & 0.077$\pm$0.025 & 6.90& 0.202$\pm$0.023& 7.27\\
    \noalign{\smallskip}\hline\noalign{\smallskip}
    GIN+CycleNet & \textbf{0.036$\pm$0.005} & 1.91 & 0.176$\pm$0.010 & 5.51\\
    GIN+CycleNet-PEOI & 0.062$\pm$0.015 & 1.47& \textbf{0.141$\pm$0.004} & 2.02\\
    \noalign{\smallskip}\hline\noalign{\smallskip}
    \end{tabular}}
    \vskip -0.2in
\end{table}


\subsection{Existing benchmarks}
\label{subsec:bench}

We evaluate CycleNet on a variety of benchmarks from works that are closely related to the Hodge Laplacian and algebraic topology. These benchmarks include the graph regression benchmark used in~\cite{lim2022sign, bodnar2021weisfeilera}, the homology localization benchmark introduced in~\cite{keros2022dist2cycle}, the superpixel classification and trajectory classification benchmark introduced in~\cite{goh2022simplicial}.

\textbf{Graph Regression.} We evaluate CycleNet on ZINC~\cite{dwivedi2020benchmarking},  a large-scale molecular dataset consisting of 12k graphs for drug-constrained solubility prediction. The evaluation metric is the MAE between the ground truth and the prediction. Table~\ref{tab:zinc} presents the results, indicating that the proposed model, particularly the basis invariant encoding, outperforms all baseline models, showcasing its strong representational capacity. Notably, despite being theoretically stronger, the basis invariant encoding from BasisNet underperforms the sign invariant encoding from SignNet, demonstrating that balancing computational efficiency, algorithmic robustness, and theoretical representational power among all eigenspaces may pose serious challenges for the former encoding. 

\textbf{Homology Localization.} We evaluate CycleNet on a dataset consisting of Alpha complexes~\cite{edelsbrunner2011alpha} that arise from "snapshots" of filtrations~\cite{edelsbrunner2022computational} on a point cloud data sampled from tori manifolds. The dataset comprises 400 point cloud graphs with the number of holes ranging from 1 to 5. The task is to predict the distance from each edge to its nearest cycle, and the mean squared error (MSE) is adopted as the evaluation metric. We add the basis-invariant encoding on the backbone model, adopt the default 6 families of datasets and report the results in Table~\ref{tab:dist}. We observe that CycleNet successfully prevents large-scale error and reaches the best performance on most benchmarks. This demonstrates the strong representation power of the basis-invariant embedding.


\begin{table}
	\centering
	\caption{Experiments on Homology Localization, MSE}  
	\scalebox{1.0}{
	\begin{tabular}{lccccccc}
	\hline\noalign{\smallskip}
	Family & 1 & 2 & 3 & 4 & 5 & 6 & Mean\\
	\noalign{\smallskip}\hline\noalign{\smallskip}
	Dist2Cycle & \textbf{0.203} & 0.105 & \textbf{0.093} & 0.524 & 0.110 & 0.108 &  0.190\\
	Dist2Cycle+CycleNet & 0.284 & \textbf{0.091} & 0.117 & \textbf{0.076} & \textbf{0.071} & \textbf{0.091} & \textbf{0.122}\\
	\noalign{\smallskip}\hline\noalign{\smallskip}
    \label{tab:dist}
	\end{tabular}}
 \vspace{-.2in}
\end{table}

\textbf{Superpixel Classification.} Following the settings of~\cite{goh2022simplicial}, we construct a superpixel graph dataset from MNIST~\cite{mnist}, an image classification dataset that contains handwritten digits from 0 to 9, using the Simple Linear Iterative Clustering (SLIC) algorithm~\cite{achanta2012slic}. In this dataset, pixels are grouped into nodes representing perceptually meaningful regions, and the resulting graph contains high-order structures such as triangles. We add the basis-invariant encoding to the backbone model SAT and report the classification accuracy as the evaluation metric. The results, presented in Table~\ref{tab:sat}, show that CycleNet outperforms all baseline methods, demonstrating that the added cycle information can not only enhance the performance on simple graphs but also contribute to graphs that contain high-order structures. It is worth noting that since these superpixel graphs are built upon simplicial complexes, CWN with cell complex-based representation works the same as simplicial graph networks.

\textbf{Trajectory classification.} In accordance with the experimental setup of~\cite{goh2022simplicial}, we present the trajectory classification dataset, which is a dense point cloud dataset containing 1000 points. Trajectories are formed by randomly selecting a starting point from a specified corner and an endpoint from another corner, each with different orientations, and the goal is to classify the type of trajectory. Table~\ref{tab:sat} shows that CycleNet surpasses all baseline methods, demonstrating the strong power of the basis-invariant encoding on high-order graphs.

\textbf{Discussion on the choice of cycle information.} In summary, we find that CycleNet-PEOI is a more efficient and comparable alternative to CycleNet for many synthetic and real-world benchmarks. However, for high-order graphs, high-order structures such as triangles and cells replace the presence of cycles,  leading to a loss of essential information when only encoding cycles. Thus, CycleNet-PEOI may not be suitable in such situations. In addition, we include an ablation study in the appendix, where we replace the encoding of the cycle space of the Hodge Laplacian with the original Hodge Laplacian. Our experiments show that the cycle space of the Hodge Laplacian contributes more to graph representation learning.

\begin{minipage}{\textwidth}
 \begin{minipage}[t]{0.41\textwidth}
  \centering
     \makeatletter\def\@captype{table}\makeatother\caption{Evaluation on ZINC (MAE).}
     \scalebox{1.0}{
       \begin{tabular}{lc}
	\hline\noalign{\smallskip}
	Method & MAE \\
	\noalign{\smallskip}\hline\noalign{\smallskip}
	GIN & 0.220  \\
	PNA  & 0.145 \\
	BasisNet & 0.094 \\
	SignNet & 0.084\\
	CWN & 0.079\\

\noalign{\smallskip}\hline\noalign{\smallskip}

	SignNet+CycleNet & 0.078 \\
        SignNet+CycleNet-PEOI & 0.082\\
        CWN+CycleNet & \textbf{0.068}\\
    \noalign{\smallskip}\hline\noalign{\smallskip}
    \label{tab:zinc}
	\end{tabular}}
  \end{minipage}
  \begin{minipage}[t]{0.59\textwidth}
   \centering
        \makeatletter\def\@captype{table}\makeatother\caption{Evaluation on superpixel classification and trajectory classification (Classification Accuracy).}
        \scalebox{1.0}{
         \begin{tabular}{lcc}
	   \hline\noalign{\smallskip}
		Method & Super & Traj\\
			\noalign{\smallskip}\hline\noalign{\smallskip}
		GCN & 63.65$\pm$1.82 & -\\
		GAT & 88.95$\pm$0.99 & - \\
		SCN & 84.16$\pm$1.23 & 52.80$\pm$3.11\\
		SCCONV & 89.06$\pm$0.47 & 62.30$\pm$ 3.97\\
		SAT & 92.99$\pm$0.71 & 93.80$\pm$1.33\\
            \noalign{\smallskip}\hline\noalign{\smallskip}
		SAT+CycleNet & \textbf{93.97$\pm$0.57} & \textbf{95.60$\pm$2.64}\\
		\noalign{\smallskip}\hline\noalign{\smallskip}
\label{tab:sat}

	\end{tabular}}
   \end{minipage}
   \vspace{-.1in}
\end{minipage}

\section{Conclusion}
To effectively incorporate the cycle information to graph learning models, we propose CycleNet, a framework that encodes the cycle space of the Hodge Laplacian in a basis-invariant manner. To improve efficiency and intuitiveness, we also present a permutation equivariant and order invariant encoding based on the theory of algebraic topology. We theoretically analyze the expressiveness of the model in terms of distinguishing non-isomorphic graphs, and empirically evaluate the model using various tasks and benchmarks. The results demonstrate  that CycleNet achieves a satisfying representation power while maintaining high algorithmic efficiency.

\nocite{langley00}

\bibliography{example_paper}
\bibliographystyle{plain}

\newpage
\appendix
\onecolumn


\section*{Appendix}

\section{Proof of Theorem~\ref{th:cyclenet} in the main paper}


We begin by introducing the \textit{graph isomorphism}. For a pair of graphs $G_1 = (V_1, E_1)$ and $G_2 = (V_2, E_2)$, if there exists a bijective mapping $f: V_1 \rightarrow V_2$, so that for any edge $(u_1, v_1) \in E_1$, it satisfies that $(f(u_1), f(v_1)) = (u_2, v_2) \in E_2$, then $G_1$ is isomorphic to $G_2$, otherwise they are not isomorphic. Up to now, there is no polynomial algorithm for solving the graph isomorphism problem. One popular method is to use the $k$-order Weisfeiler-Leman~\cite{weisfeiler1968reduction} algorithm ($k$-WL). It is known that  1-WL is as powerful as 2-WL, and for $k \geq 2$, $(k+1)$-WL is more powerful than $k$-WL. 

We then provide the theoretical results below:

\begingroup
\def\thetheorem{\ref{th:cyclenet}}

\begin{theorem}
CycleNet is strictly more powerful than 2-WL, and can distinguish graphs that are not distinguished by 3-WL.
\end{theorem}
\addtocounter{theorem}{-1}
\endgroup

\begin{proof}

\textbf{The pair of graphs that 3-WL cannot distinguish while CycleNet can.} It is shown in~\cite{arvind2020weisfeiler} that 3-WL cannot differentiate the $4 \times 4$ Rook Graph and the Shrikhande Graph shown in Figure~\ref{fig:3wl}. We then compute the orthogonal projector of the cycle space of the Hodge Laplacian for each graph and denote them as $O_{rook}$ and $O_{sh}$. We observe that each column of $O_{rook}$ contains 22 zeros, whereas each column of $O_{sh}$ contains 16 zeros. To differentiate between the two graphs, we can use the function $|O_{rook} - O_{sh}|$, which can be approximated using an invariant graph network (IGN) followed by a multilayer perceptron (MLP). Specifically, the 2-2 layer of the IGN can obtain the $O_{rook}$ and $O_{sh}$, and the MLP can approximate the absolute function.

\textbf{More powerful than the 2-WL.} Using models such as~\cite{xu2018powerful} to be the backbone GNNs can distinguish any pair of non-isomorphic graphs that 2-WL can distinguish. Since there exist graphs such as the $4\text{x}4$ Rook Graph and the Shrikhande graph that 2-WL cannot distinguish, while CycleNet can. Therefore, CycleNet is more powerful than 2-WL.
\end{proof}

\begin{figure}[btp]
	\centering
	\subfigure[]{
		\begin{minipage}[t]{0.52\linewidth}
			\centering
			\includegraphics[width=0.95\columnwidth]{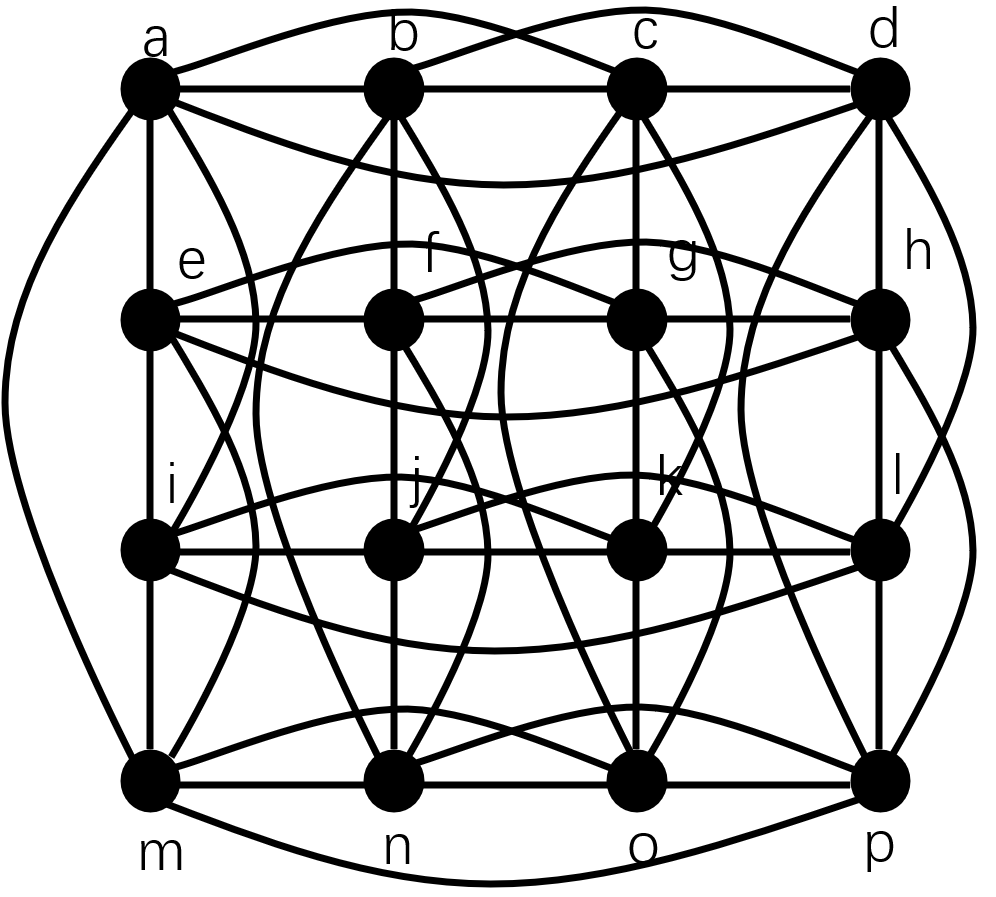}
		\end{minipage}%
	}%
	\subfigure[]{
		\begin{minipage}[t]{0.48\linewidth}
			\centering
			\includegraphics[width=0.95\columnwidth]{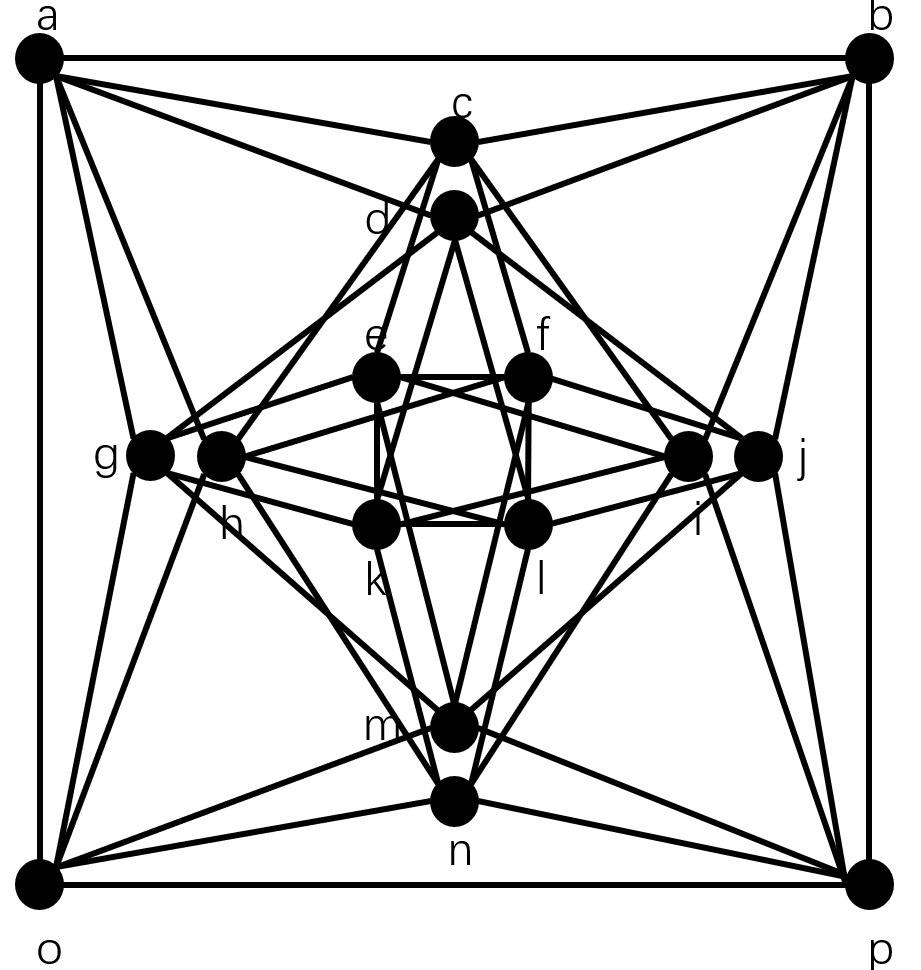}
		\end{minipage}%
	}%
	\centering
	\vspace{-.1in}
	\caption{(a) the $4\text{x}4$ Rook Graph and (b) the Shrikhande Graph}
	\label{fig:3wl}
	\vspace{-.15in}
\end{figure}

\section{Proof of Theorem~\ref{th:cyclenet_scb} in the main paper}

We restate the theorem as follows:

\begingroup
\def\thetheorem{\ref{th:cyclenet_scb}}

\begin{theorem}
Using the length of shortest cycle basis as the edge structural embedding can distinguish certain pair of graphs that are not distinguished by 3-WL, as well as pair of graphs that are not distinguished by 4-WL.
\end{theorem}
\addtocounter{theorem}{-1}
\endgroup

\begin{proof}

\textbf{The pair of graphs that 4-WL cannot distinguish.} Consider the set of graphs called the Cai-F\"urer-Immerman (CFI) graphs~\cite{cai1992optimal}. The sequence of graphs $G_k^{(\ell)}, \ell = 0, 1, \ldots, k+1$ is defined as following,
	\begin{equation}
    \label{eq:cfi}
		\begin{split}
			V_{G_k^{(\ell)}} =& \Big\{u_{a, \vec{v}}\Big| a\in[k+1], \vec{v}\in\{0,1\}^k \text{ and } \vec{v} \text{ contains }\\
			&\quad
            \begin{array}{ll}
				\text{an even number of } 1 \text{'s}, & \text{if }a=1,2,\ldots, k-\ell+1, \\
				\text{an odd number of } 1 \text{'s},  & \text{if }a=k-\ell+2,\ldots, k+1.
			\end{array}\Big\}
		\end{split}
	\end{equation}
	Edges exists between two nodes $u_{a,\vec{v}}$ and $u_{a',\vec{v}'}$ of $G_k^{(\ell)}$ if and only if there exists $m\in [k]$ such that $a' \mod (k+1) = (a+m) \mod (k+1)$ and $v_m = v'_{k-m+1}$. 

Denote the two graphs $G=G_4^{(0)}$ and $H=G_4^{(1)}$. It is shown in~\cite{yan2023efficiently} that 4-WL cannot differentiate the pair of graphs.


 

 \textbf{The SCB can distinguish them.} We begin by presenting the computation of the shortest cycle basis. Let $C_T \in R^{m \times l}$ denote the set of all tight cycles, where $m$ is the number of edges and $l$ is the number of tight cycles. The definition of tight cycles is described in Section 3.3 of the main paper. For a given cycle $j$, $C_T[i][j]$ is equal to 1 if edge $i$ is in cycle $j$, and 0 otherwise. We define $low_{C_T}(j)$ as the maximum row index $i$ such that $C_T[i][j] = 1$. To compute the shortest cycle basis, we use the matrix reduction algorithm, which is shown in Algorithm~\ref{alg:MR}.

 \begin{algorithm}[h]
	\caption{Matrix Reduction}
	\label{alg:MR}
	\begin{algorithmic}
		\STATE {\bfseries Input:} the set of tight cycles $C_T$
		\STATE the shortest cycle basis $SCB=\{\}$
		\STATE$C_T = \text{SORT}(C_T)$
		\FOR{$j=1$ {\bfseries to} $l$}
		\WHILE{$\exists k < j$ with $low_{C_T}(k)=low_{C_T}(j)$}
		\STATE add column $k$ to column $j$ and 
		\ENDWHILE
             \IF{column $j$ is not a zero vector}
		\STATE add the original column $j$ to $SCB$
            \ENDIF
		\ENDFOR
		\STATE {\bfseries Output:} the shortest cycle basis $SCB$
	\end{algorithmic}
\end{algorithm}

In the given algorithm, the symbol ``add" represents the modulo-2 sum of two binary vectors. It should be noted that Algorithm~\ref{alg:MR} may not be the fastest algorithm for computing the SCB, but most acceleration methods are based on it. The algorithm processes the cycles in $C_T$ in order of increasing length, with shorter cycles added to the shortest cycle basis before longer cycles. If any cycle can be represented as a sum of multiple cycles whose lengths are no more than $k$, then the length of the longest cycle in the shortest cycle basis will be $k$. We denote a cycle with length $k$ as a $k$-cycle.

We obtain a total of 40 nodes for $G$ and $H$ by traversing $a$ from 1 to 5 according to Equation~\ref{eq:cfi}. For example, in $G$, node $1$ denotes $u_{1, \{0,0,0,0\}}$, and node $2$ denotes $u_{1, \{0,0,1,1\}}$. We then traverse these nodes to obtain the edges. For example, edge $1$ denotes $(1, 9)$ in $G$, which corresponds to node $u_{1, \{0,0,0,0\}}$ and node $u_{2, \{0,0,0,0\}}$. It is observed that in $H$, a 4-cycle exists between edges $\{8, 9, 24, 25\}$. These edges correspond to four nodes: $u_{1, \{0,0,0,0\}}$, $u_{1, \{0,0,1,1\}}$, $u_{4, \{0,0,0,0\}}$, and $u_{4, \{0,1,0,1\}}$. The 4-cycle cannot be represented by the modulo-2 sum of 3-cycles since there is no 3-cycle whose edge with the maximum index after matrix reduction borns earlier than edge 25, that is ($u_{1, \{0,0,1,1\}}$, $u_{4, \{0,1,0,1\}}$). Therefore the SCB of $H$ contains 4-cycle.

The same 4-cycle also exists in $G$, and it can be represented by 38 3-cycles:  $\{12, 288, 8\}$, $\{89,  94, 296\}$, $\{12, 148, 1\}$, $\{ 23, 215,  19\}$, $\{105, 108, 301\}$, $\{23, 282,  27\}$, $\{218, 282, 215\}$, $\{195, 318, 199\}$, $\{195, 316, 198\}$, $\{103, 105, 267\}$, $\{9, 144,   1\}$, $\{115, 121, 217\}$, $\{115, 124, 220\}$, $\{218, 313, 220\}$, $\{234, 314, 236\}$, $\{121, 124, 301\}$, $\{147, 318, 151\}$, $\{147, 316, 150\}$, $\{146, 314, 149\}$, $\{146, 312, 148\}$, $\{170, 234, 165\}$, $\{170, 313, 172\}$, $\{192, 198, 296\}$, $\{192, 199, 298\}$, $\{99, 108, 172\}$, $\{213, 267, 217\}$, $\{99, 101, 165\}$, $\{97, 103, 141\}$, $\{97, 109, 149\}$, $\{213, 266, 216\}$, $\{81,  89, 144\}$, $\{81,  94, 150\}$, $\{19, 216,  25\}$, $\{266, 270, 298\}$, $\{101, 109, 236\}$, $\{141, 270, 151\}$, $\{28, 312,  27\}$, $\{28, 288,  24\}$. The same situations exist for all other 4-cycles or cycles longer than 4. We also observe that there have been 281 3-cycles in the SCB. Considering that it is equal to the Betti number of $G$, the SCB does not contain any 4-cycle.

\textbf{The pair of graphs that 3-WL cannot distinguish while SCB can.} There are 24 3-cycles and 9 4-cycles in the SCB of the $4 \times 4$ Rook Graph, while there are 31 3-cycles and 2 4-cycles in the SCB of the Shrikhande Graph. Therefore the SCB can differentiate them.

\end{proof}

\begin{figure}[btp]
	\centering
	\subfigure[]{
		\begin{minipage}[t]{0.5\linewidth}
			\centering
			\includegraphics[width=0.95\columnwidth]{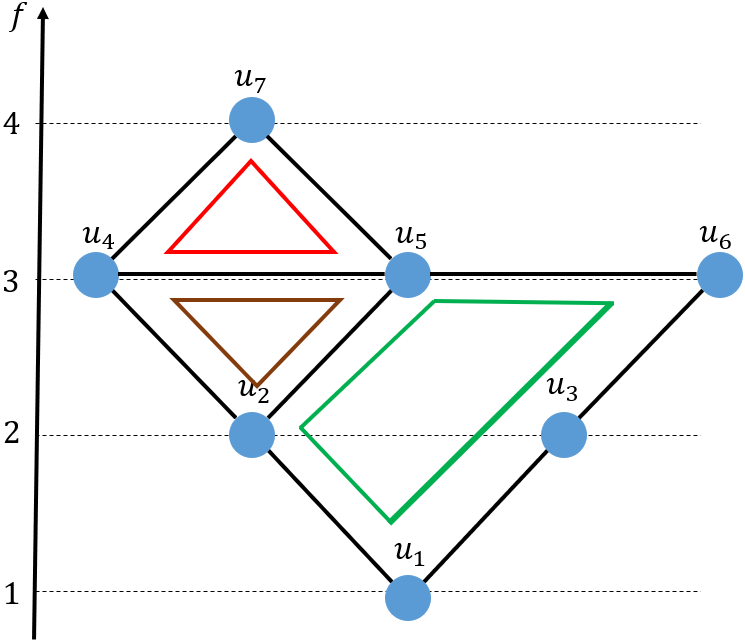}
		\end{minipage}%
	}%
	\subfigure[]{
		\begin{minipage}[t]{0.5\linewidth}
			\centering
			\includegraphics[width=0.95\columnwidth]{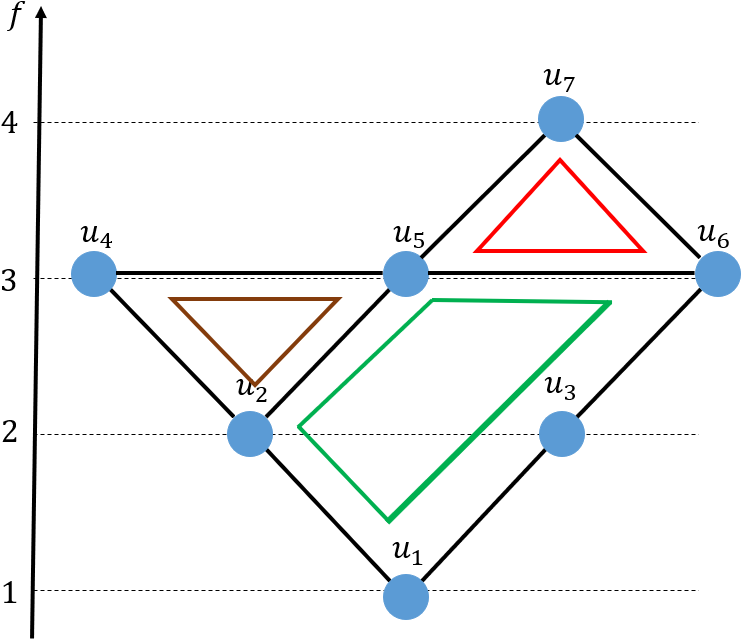}
		\end{minipage}%
	}%
	\centering
	\vspace{-.1in}
	\caption{Graphs that the PEOI encoding of the cycle incidence matrix can differentiate, while the number of cycles and the extended persistence diagrams cannot.}
	\label{fig:peoi}
	\vspace{-.15in}
\end{figure}


\section{Proof of Theorem~\ref{th:number} in the main paper}
\label{subsec:number}

We restate the theorem as follows:
\begingroup
\def\thetheorem{\ref{th:number}}

\begin{theorem}
If choosing the same set of cycles. The PEOI encoding of the cycle incidence matrix is more powerful than using its number in terms of distinguishing non-isomorphic graphs.
\end{theorem}
\addtocounter{theorem}{-1}
\endgroup

\begin{proof}

\textbf{PEOI can extract the number of cycles.} In Proposition 4.2 in the main paper, if we set $\rho_1$ as a function that consistently produces ``1", $\rho_2$ as a function that ignores the  $X[i][k]$ element while being an identity function for the rest elements, and $\rho_3$ as an identity function, we can obtain the number of cycles. Therefore, the PEOI encoding of the cycle incidence matrix is at least as powerful as the number of cycles.

Then we use the pair of graphs shown in Figure~\ref{fig:peoi} as an example. 

\textbf{The number of cycles cannot differentiate the pair of graphs.} In these two graphs the number of cycles will remain the same. For example, if using all the cycles, there are both 3 cycles in Figure~\ref{fig:peoi}(a) and Figure~\ref{fig:peoi}(b). If using cycles of a certain length, there are both 2 3-cycles and 1 5-cycle in Figure~\ref{fig:peoi}(a) and Figure~\ref{fig:peoi}(b). Therefore, only using the number of cycles cannot differentiate the pair of graphs.

\textbf{The PEOI encoding of the cycle incidence matrix can differentiate the pair of graphs.} The cycle incidence matrix of these two graphs is listed as follows:
\definecolor{emerald}{rgb}{0.31, 0.78, 0.47}
\begin{minipage}[t]{0.5\linewidth}
\begin{center}
$\begin{bmatrix}
 & \textcolor{emerald}{\gamma_g} & \textcolor{brown}{\gamma_b} & \textcolor{red}{\gamma_r} \\
(u_1, u_2) & 1 & 0 & 0 \\
(u_1, u_3) & 1 & 0 & 0 \\
(u_2, u_4) & 0 & 1 & 0 \\
(u_2, u_5) & 1 & 1 & 0 \\
(u_3, u_6) & 1 & 0 & 0 \\
(u_4, u_5) & 0 & 1 & 1 \\
(u_5, u_6) & 1 & 0 & 0 \\
(u_4, u_7) & 0 & 0 & 1 \\
(u_5, u_7) & 0 & 0 & 1\\
\end{bmatrix}$
\end{center}
\end{minipage}%
\begin{minipage}[t]{0.5\linewidth}
\begin{center}
$\begin{bmatrix}
 & \textcolor{emerald}{\gamma_g} & \textcolor{brown}{\gamma_b} & \textcolor{red}{\gamma_r} \\
(u_1, u_2) & 1 & 0 & 0 \\
(u_1, u_3) & 1 & 0 & 0 \\
(u_2, u_4) & 0 & 1 & 0 \\
(u_2, u_5) & 1 & 1 & 0 \\
(u_3, u_6) & 1 & 0 & 0 \\
(u_4, u_5) & 0 & 1 & 0 \\
(u_5, u_6) & 1 & 0 & 1 \\
(u_5, u_7) & 0 & 0 & 1 \\
(u_6, u_7) & 0 & 0 & 1\\
\end{bmatrix}$
\end{center}
\end{minipage}%

For Proposition 4.2 in the main paper, we can define $\rho_1 (X[i][k], X[j][k]) = 2X[i][k] + X[j][k]$, $\rho_2(X[i][k], Y) = RELU(Y - 16)$, and $\rho_3$ to be an identity function. Therefore, for the 
graph shown in Figure~\ref{fig:peoi}(a), the PEOI encoding is $\{4, 4, 2, 6, 4, 4, 4, 2, 2\}$; for the graph shown in Figure~\ref{fig:peoi}(b), the PEOI encoding is $\{4, 4, 2, 6, 4, 2, 6, 2, 2\}$. According to Proposition 4.1 in the main paper, we can differentiate the pair of graphs using CycleNet-PEOI.

Therefore, the PEOI encoding of the cycle incidence matrix is more powerful than the number of cycles.

\end{proof}

\section{Proof of Theorem 4.5 in the main paper}
\label{subsec:ph}

The classic EPDs~\cite{cohen2009extending} can be used to measure the saliency of connected components and high-order topological structures such as voids. However, recent works~\cite{yan2021link, yan2022neural, zhang2022gefl} have mainly used the one-dimensional (1D) EPD as augmented topological features, particularly the features corresponding to cycles. Therefore, in this section, we mainly focus on comparing our encoding with the \textbf{1D EPDs corresponding to cycles}. For ease of complexity, we will omit the terms ``1D" and ``that correspond to cycles" in the rest of this section, and only use ``EPDs".

\begingroup
\def\thetheorem{\ref{th:ph}}

\begin{theorem}
If choosing the same set of cycles. The PEOI encoding of the cycle incidence matrix can differentiate graphs that cannot be differentiated by the extended persistence diagram. If adding the filter function to the cycle incidence matrix, the PEOI encoding of the cycle incidence matrix is more powerful than using its extended persistence diagram in terms of distinguishing non-isomorphic graphs.
\end{theorem}
\addtocounter{theorem}{-1}
\endgroup

\begin{proof}

\textbf{The extended persistence diagram (EPD).}  Persistent homology~\cite{edelsbrunner2022computational, edelsbrunner2000topological} captures topological structures such as connected components and cycles, and summarizes them in a point set called the persistence diagram (PD). It is found that the extended persistence diagrams (EPD)~\cite{cohen2009extending} is a variant of PD that encodes richer cycle information.  Specifically, an EPD is a set of points in which every point represents the significance of a topological structure in terms of a scalar function known as the filter function. Recent studies have shown that the extended persistence point of a cycle is the combination of the maximum and minimum filter values of the point in the cycle~\cite{yan2022neural}. Note that in this paper, we focus on the EPDs of cycles, and do not consider the EPDs of other structures.

We illustrate the computation of the EPD for the graph in Figure~\ref{fig:peoi}(a),  where the filter function is defined as the shortest path distance from a selected root node $u_1$ to other nodes. This is a common filter function used in previous models~\cite{zhao2020persistence, yan2021link}. We plus one to the filter value in case of the zero value. Using this definition, we have: $f(u_1) = 1$, $f(u_2) = f(u_3) = 2$, $f(u_4) = f(u_5) = f(u_6) = 3$, and $f(u_7) = 4$. The extended persistence point of the red cycle, the brown cycle, and the green cycle are $(3, 1)$, $(3, 2)$, and $(4, 3)$, respectively. Therefore the EPD of Figure~\ref{fig:peoi}(a) is $\{(3, 1), (3, 2), (4, 3)\}$. We can define similarly the filter value for Figure~\ref{fig:peoi}(b), and the extended persistence points of the three cycles are the same as the cycles in Figure~\ref{fig:peoi}(a). Therefore, the EPD of Figure~\ref{fig:peoi}(b) is also $\{(3, 1), (3, 2), (4, 3)\}$, and the EPD cannot differentiate the pair of graphs. Note that the PEOI encodings for these two graphs are different, as shown in the proof of Theorem~\ref{th:number}.

\textbf{Add the filter function to CycleNet-PEOI.} It is worth noting that the filter function, which plays a crucial role in constructing the EPD, is not explicitly contained in the cycle incidence matrix. As a result, encoding the original cycle incidence matrix using the proposed PEOI method is not sufficient to extract the EPD.

However, we can incorporate the filter function into the proposed model by adding the filter function to the cycle incidence matrix. For example, we define the filter value of an edge as the minimum value of the nodes in the edge, and obtain the filter values of the edges in Figure~\ref{fig:peoi}(a) as $\{1, 1, 2, 2, 2, 3, 3, 3, 3\}$. Next, we compute the dot product between the filter values of edges and the cycle incidence matrix, which results in the so-called filter-enhanced cycle incidence matrix. Similarly, we can obtain the filter-enhanced cycle incidence matrix for Figure~\ref{fig:peoi}(b). The two matrices are listed below:
\definecolor{emerald}{rgb}{0.31, 0.78, 0.47}
\begin{minipage}[t]{0.5\linewidth}
\begin{center}
$\begin{bmatrix}
 & \textcolor{emerald}{\gamma_g} & \textcolor{brown}{\gamma_b} & \textcolor{red}{\gamma_r} \\
(u_1, u_2) & 1 & 0 & 0 \\
(u_1, u_3) & 1 & 0 & 0 \\
(u_2, u_4) & 0 & 2 & 0 \\
(u_2, u_5) & 2 & 2 & 0 \\
(u_3, u_6) & 2 & 0 & 0 \\
(u_4, u_5) & 0 & 3 & 3 \\
(u_5, u_6) & 3 & 0 & 0 \\
(u_4, u_7) & 0 & 0 & 3 \\
(u_5, u_7) & 0 & 0 & 3\\
\end{bmatrix}$
\end{center}
\end{minipage}%
\begin{minipage}[t]{0.5\linewidth}
\begin{center}
$\begin{bmatrix}
 & \textcolor{emerald}{\gamma_g} & \textcolor{brown}{\gamma_b} & \textcolor{red}{\gamma_r} \\
(u_1, u_2) & 1 & 0 & 0 \\
(u_1, u_3) & 1 & 0 & 0 \\
(u_2, u_4) & 0 & 2 & 0 \\
(u_2, u_5) & 2 & 2 & 0 \\
(u_3, u_6) & 2 & 0 & 0 \\
(u_4, u_5) & 0 & 3 & 0 \\
(u_5, u_6) & 3 & 0 & 3 \\
(u_5, u_7) & 0 & 0 & 3 \\
(u_6, u_7) & 0 & 0 & 3\\
\end{bmatrix}$
\end{center}
\end{minipage}%

\textbf{Define the PEOI encoding.} We can use a 2-layer MLP to approximate the minimum function between two elements. The hidden layer contains 4 nodes, and the ReLU activation function is used. The weights from the input layer to the hidden layer are $(1, 1)$, $(1, -1)$, $(-1, 1)$, and $(-1, -1)$, respectively, and the biases are set to 0 for all nodes. The weights from the hidden layer to the output layer are 0.5, -0.5, -0.5, -0.5, respectively.

In Proposition 4.2 in the main paper, We set $\rho_1$ as the minimum function, which is approximated by the 2-layer MLP. $\rho_2$ is defined as a function that ignores the  $X[i][k]$ element while being an identity function for another element. $\rho_3$ is set as an identity function.

\textbf{The defined encoding can differentiate the graphs that EPD can also differentiate.} Assume that there exists a pair of graphs $G_1$ and $G_2$ whose EPDs are different, we can assume that there exist a pair of cycles whose lowest filter values are different (the highest filter values can be treated similarly). Under this assumption, we can define the filter value for edges as the minimum value of the nodes in the edge. Using the PEOI encoding defined above, we can extract the lowest filter value of these two cycles. We then use an injective function on the multiset of cycle embeddings to produce different outputs for these two graphs. Therefore the defined encoding can differentiate $G_1$ and $G_2$.



In conclusion, by incorporating the filter function, CycleNet-PEOI can differentiate all pairs of graphs that the EPDs can differentiate, and can distinguish graphs that EPDs cannot. Therefore, it is more powerful than EPDs in terms of distinguishing non-isomorphic graphs.

\end{proof}

\section{Implementation details}

\textbf{Encoding of CycleNet-PEOI.} Based on Proposition 4.2 in the main paper, we provide a pytorch-like pseudo-code for the PEOI encoding in Figure~\ref{fig:code}. 

\lstset{
  backgroundcolor=\color{white},
  basicstyle=\fontsize{7.5pt}{8.5pt}\fontfamily{lmtt}\selectfont,
  columns=fullflexible,
  breaklines=true,
  captionpos=b,
  commentstyle=\fontsize{8pt}{9pt}\color{blue},
  keywordstyle=\fontsize{8pt}{9pt}\color{purple},
  stringstyle=\fontsize{8pt}{9pt}\color{blue},
  frame=tb,
  otherkeywords = {self},
}

\begin{figure}[ht]
    \centering
    \begin{minipage}{\columnwidth}
    \begin{lstlisting}[language=python,
        title={PyTorch-like pseudo-code for PEOI encoding},
        captionpos=t]
    class PEOI(nn.Module):
    
        def __init__(self, D1, D2, D3):
            self.rho1 = MLP(2, D1) # in dim=2, out dim=D1
            self.rho2 = MLP(1+D1, D2) # in dim=1+D1, out dim=D2
            self.rho3 = MLP(D2, D3) # in dim=D2, out dim=D3
            
        def forward(self, x):
            # x shape: m x g
            
            m, g = x.shape
            x1 = x.reshape(m, 1, 1, g).expand(-1, m, 1, -1) # m x m x 1 x g
            x2 = x.reshape(1, m, 1, g).expand(m, -1, 1, -1) # m x m x 1 x g
            w = self.rho1(cat((x1; x2), dim = 2)).sum(dim = 1) # m x m x 2 x g -> m x D1 x g
            w = self.rho2(cat((x.reshape(m, 1, g); w), dim = 1)).sum(dim = 2) # m x (1+D1) x  g -> m x D2
            w = self.rho3(w) # m x D2 -> m x D3
            
            return w
    \end{lstlisting}
    \end{minipage}
    \caption{PyTorch-like pseudo-code for PEOI encoding. Here ``cat((x, y), dim = c)" denotes the concatenation of two matrices on the c-th dimension. ``sum(dim=c)" denotes the sum operation over the c-th dimension.}
    \label{fig:code}
\end{figure}

In certain situations where graphs are dense and large, the original PEOI encoding may bring extra computational and memory costs. In these situations, we can ignore the final $X[i][k]$ element in Proposition 4.2 in the main paper, and then the memory cost will be no larger than $O(m \times g)$.

\textbf{Encoding of CycleNet.} The full approximation power requires high-order tensors to be used for the IGN~\cite{maehara2019simple, maron2019universality, keriven2019universal}. In practice, we follow the settings of~\cite{lim2022sign} and restrict the tensor dimensions for efficiency. This encoding, although losing certain theoretical power, shows strong empirical performance in~\cite{lim2022sign}.

\textbf{Experimental details.} In the synthetic experiments in the main paper, we use a 5-layer GIN~\cite{xu2018powerful} as the backbone model. We set the hidden dimension to 128, batch size to 16, and learning rate to 1e-3 with Adam as the optimizer. We use a ReduceLROnPlateau scheduler with the reduction factor set to 0.7,  the patience set to 10, and the minimum learning rate set to 1e-6. In the synthetic experiments related to cycles, we use a point cloud dataset sampled on small cycles whose centers are on a big cycle. The diameters of the large cycle and small cycle are set to 20 and 1, respectively. We randomly sample 20 points from the large cycle and 60 points from the small cycle. After obtaining the node set, we generate a k-nearest-neighbor graph with the parameter k set to 3. There is no input feature for the prediction of the Betti Number. As for the prediction of EPD, we use the position of the node as the filtration function of the EPD. The input node feature is therefore the coordinates of the nodes.

For real-world benchmarks, we use SignNet or CWN as the backbone model on ZINC. However, models such as CWN will build the cell complex (or simplicial complex) on cycles. Therefore using the proposed cycle-invariant structural encoding is not a good choice since many of these features are already filled with the cells. Instead, we use the original Hodge Laplacian as the input of 2-IGN, which is also cycle-invariant. Our settings follow exactly the settings of SignNet or CWN. For the superpixel classification and the trajectory classification benchmarks, we use SAT as the backbone model. Our settings follow exactly the settings of SAT. For the homology localization benchmark, we use Dist2cycle as the backbone model. Our settings follow exactly the settings of Dist2cycle. Notice that for backbone models that fill the cycles with 2-cells, the kernel space of the Hodge Laplacian may not contain any information. Therefore, we replace the kernel space encoding with the encoding based on the original Laplacians. All the experiments are implemented with two Intel Xeon Gold 5128 processors,192GB RAM, and 10 NVIDIA 2080TI graphics cards.

\textbf{The assets we used.} Our model is experimented on benchmarks from~\cite{dwivedi2020benchmarking, goh2022simplicial, keros2022dist2cycle, mnist, bodnar2021weisfeiler, balcilar2021breaking} under the MIT license.

\textbf{Limitations of the paper.} First, we have shown that the representation power of our model is bounded by high-order WLs in terms of distinguishing non-isomorphic graphs.

Second, the proposed model may not perform well on benchmarks where cycle information is not relevant. For example, in high-order graphs where cycles are replaced by high-order structures like triangles or cells, the proposed CycleNet-PEOI model may not be suitable.







\section{Additional experiments}


\subsection{Ablation study on ZINC} 

We present additional evaluations on (1) the memory cost in terms of the number of trainable parameters; (2) the effectiveness of the introduced cycle-related embedding on a wider range of settings; and (3) the comparison between the original Hodge Laplacian and the cycle space of the Hodge Laplacian.

To conduct the evaluation, we follow the settings of~\cite{lim2022sign} and report the results in Table~\ref{tab:ablation}. Specifically, we name the framework \textbf{CycleNet-Hodge}, which replaces the orthogonal projector of the cycle space of the Hodge Laplacian with the original Hodge Laplacian. Notably, we follow the implementation of IGN in~\cite{lim2022sign}, which restricts the tensor dimensions for efficiency, leading to a slight theoretical limitation but strong empirical performance. 

We find from the table that the proposed cycle-related information improves the performance of all backbones while only adding a few extra learnable parameters. This provides empirical evidence that the proposed structural embedding is robust across different backbone models. Additionally, CycleNet outperforms CycleNet-Hodge across all backbones, indicating that the basis-invariant encoding of the cycle space is better at extracting useful cycle-related information. This is potentially because 2-IGN cannot effectively model the matrix multiplication or rank computation of 2D matrices. While the original Hodge Laplacian encodes the information of the cycle space, 2-IGN may fail to extract the information. We also observe that BasisNet introduces too many additional parameters and performs worse than our model, demonstrating the trade-off between computational efficiency and theoretical representation power when generating a basis-invariant encoding for all eigenspaces. Furthermore, comparing CWN to CycleNet, CWN achieves comparable results with CycleNet and CycleNet-PEOI, indicating its strong representation power. However, CWN introduces too many trainable parameters, leading to high memory and computational costs.

\begin{table}
		\centering
		\caption{Additional experiments on ZINC}  
        \label{tab:ablation}
		\scalebox{1.0}{
			\begin{tabular}{lcc}
					\hline\noalign{\smallskip}
					Framework & test MAE & Params\\
					\noalign{\smallskip}\hline\noalign{\smallskip}
                    GIN & 0.220 &  497394\\
                    + CycleNet-Hodge &  0.165 & 543876   \\
					+ CycleNet &  0.153 &  543876  \\
					+ CycleNet-PEOI & 0.153 &  512812  \\
					+ BasisNet & 0.169 & 751810 \\
     			\noalign{\smallskip}\hline\noalign{\smallskip}
					GatedGCN & 0.259 & 491597\\
                    + CycleNet-Hodge &  0.142 & 510539  \\
					+ CycleNet & 0.137 & 510539  \\
					+ CycleNet-PEOI &  0.188 & 504090\\
					+BasisNet & 0.139 & 716793 \\
     			\noalign{\smallskip}\hline\noalign{\smallskip}
					PNA & 0.145 & 473681\\
					+ CycleNet-Hodge  & 0.128 & 479081\\
					+ CycleNet &  0.089 & 479081 \\
					+ CycleNet-PEOI &  0.111 & 483769\\
					+ BasisNet & 0.094 & 556323 \\
     			\noalign{\smallskip}\hline\noalign{\smallskip}
					SignNet & 0.084 & 487082\\
					+ CycleNet-Hodge & 0.081 & 492482\\
					+ CycleNet &  \textbf{0.077}  & 492482 \\
					+ CycleNet-PEOI &  0.082 & 497170\\
     			\noalign{\smallskip}\hline\noalign{\smallskip}
		\end{tabular}}
\end{table}

\subsection{The time to compute the eigenvectors of the Hodge Laplacian}

 In Table~\ref{tab:sta_time}, we report the statistics of ZINC and the synthetic homology dataset, including average node count and degree distribution. In addition, we report the average time (seconds) to generate the original Hodge Laplacian ("Original") and the orthogonal projector of the cycle space ("Basis") which serves as input to the basis-invariant model. Across both datasets, we find the processing step to be efficient, generating the necessary features for existing benchmark graphs in a reasonable time. The experiments are done on 64 Intel(R) Xeon(R) Gold 5218 CPUs.

\begin{table}
		\centering
		\caption{Time Evaluation on computing the Hodge Laplacian}  
        \label{tab:sta_time}
		\scalebox{1.0}{
			\begin{tabular}{lcccc}
					\hline\noalign{\smallskip}
					ZINC&	Avg. Nodes	&Avg. Degree	&Original&	Basis\\
     \hline\noalign{\smallskip}
    Train&	23.17&	2.15&	7.04e-4&	1.49e-3\\
Val	&23.08	&2.15	&7.09e-4&	1.40e-3\\
Test	&23.12&	2.15&	7.01e-4&	1.40e-3\\
\noalign{\smallskip}\hline\noalign{\smallskip}
Homology&	Avg. Nodes&	Avg. Degree	&Original&	Basis\\
\hline\noalign{\smallskip}
Train	&80&	1.17&	1.25e-3	&4.08e-3\\
Test&	80	&1.17&	1.24e-3	&4.05e-3\\

     			\noalign{\smallskip}\hline\noalign{\smallskip}
			
		\end{tabular}}
\end{table}

\section{Discussion on the uniqueness of SCB.}
\label{sec:unique}
We first report the prevalence of unique SCBs in real-world data, showing that it is reasonable to assume the uniqueness of SCBs in specific molecule graphs. Since no existing algorithm can detect whether a graph has a unique SCB, we visualize the first 100 graphs from the ZINC-12k dataset, and manually observe that all these graphs exhibit a unique SCB. This serves as empirical evidence that it is reasonable to assume a unique SCB for sparse molecule graphs.

In situations where the SCB of a graph is non-unique, the resulting feature encoding, CycleNet-PEOI, will not constitute a canonical representation. However, we can still use the encoding to capture the cycle information even if the SCB is not unique.

\end{document}